\documentclass[accepted]{uai2023} 

\usepackage[american]{babel}

\usepackage{natbib} 
\usepackage{booktabs} 
\usepackage{tikz} 

\usepackage{url}            
\usepackage{booktabs}       
\usepackage{amsfonts}       
\usepackage{nicefrac}       
\usepackage{microtype}      
\usepackage{xcolor}         
\usepackage{graphicx}
\usepackage{caption}
\usepackage{subcaption}
\usepackage{siunitx}
\usepackage{algorithm,algpseudocode}
\usepackage{wrapfig}
\usepackage{amsmath,amsthm}
\usepackage{amssymb}
\usepackage{epstopdf}
\usepackage{mathtools}
\usepackage{multirow}
\usepackage{multicol}

\usepackage{listings}
\usepackage{color}
\definecolor{mygreen}{rgb}{0,0.6,0}
\definecolor{mygray}{rgb}{0.5,0.5,0.5}
\definecolor{mymauve}{rgb}{0.58,0,0.82}
\newcommand\pythonstyle{\lstset{ 
  backgroundcolor=\color{white},   
  basicstyle=\footnotesize,        
  breakatwhitespace=false,         
  breaklines=true,                 
  captionpos=b,                    
  commentstyle=\color{mygreen},    
  deletekeywords={...},            
  escapeinside={\%*}{*)},          
  extendedchars=true,              
  firstnumber=1000,                
  frame=single,	                   
  keepspaces=true,                 
  keywordstyle=\color{blue},       
  language=Octave,                 
  morekeywords={*,...},            
  numbers=none,                    
  numbersep=5pt,                   
  numberstyle=\tiny\color{mygray}, 
  rulecolor=\color{black},         
  showspaces=false,                
  showstringspaces=false,          
  showtabs=false,                  
  stepnumber=2,                    
  stringstyle=\color{mymauve},     
  tabsize=2,	                   
  title=\lstname                   
}}

\lstnewenvironment{python}[1][]
{
\pythonstyle
\lstset{#1}
}
{}

\newcommand\pythoninline[1]{{\pythonstyle\lstinline!#1!}}

\newcommand{\MixupS}{\textit{MixupE}}
\newcommand{\inputmixup}{Mixup}

\newcommand{\nolight}{}

\definecolor{mygray}{gray}{0.4}
\newcommand{\g}[2]{#1\textsubscript{\textcolor{mygray}{$\pm$#2}}}

\usepackage{Definitions}
\newcommand{\Beta}{{\operatorname{Beta}}}

\newcommand{\bx}{{\mathbf{x}}}
\newcommand{\by}{{\mathbf{y}}}
\newcommand{\bz}{{\mathbf{z}}}
\newcommand{\bxp}{{\mathbf{x}^\prime}}

\newcommand{\cX}{{\cal X}}
\newcommand{\bR}{{\mathbb R}}
\newcommand{\cY}{{\cal Y}}
\def\tx{\tilde{\mathbf{x}}}
\def\ty{\tilde{\mathbf{y}}}

\newcommand{\bE}{\mathbb{E}}
\def\Ln{L_n^{std}}
\def\Lnmix{L^{\text{mix}}_n}

\title{MixupE: Understanding and Improving Mixup \\ from Directional Derivative Perspective}

\author[1]{Yingtian Zou \thanks{ indicates equal contribution.}}
\author[2,3]{
Vikas~Verma $^*$}
\author[2]{Sarthak~Mittal}
\author[1]{Wai~Hoh~Tang}
\author[4]{Hieu~Pham}
\author[3]{Juho~Kannala}
\author[2]{Yoshua~Bengio}
\author[3]{Arno~Solin}
\author[1]{Kenji~Kawaguchi}
\affil[1]{%
    National University of Singapore\\
    Singapore
}

\affil[2]{%
    Universite de Montreal, Mila\\
    Canada
}
\affil[3]{%
    Aalto University\\
    Finland
}
\affil[4]{%
    Google Brain\\
    USA
}

\begin{document}
\maketitle

\begin{abstract}
Mixup is a popular data augmentation technique for training deep neural networks where additional samples are generated by linearly interpolating pairs of inputs and their labels. This technique is known to improve the generalization performance in many learning paradigms and applications. In this work, we first analyze Mixup and show that it implicitly regularizes infinitely many directional derivatives of all orders. 
Based on this new insight, we propose an improved version of Mixup, theoretically justified to deliver better generalization performance than the vanilla Mixup.
To demonstrate the effectiveness of the proposed method, we conduct experiments across various domains such as images, tabular data, speech, and graphs. Our results show that the proposed method improves Mixup across multiple datasets using a variety of architectures, for instance, exhibiting an improvement over Mixup by 0.8\% in ImageNet top-1 accuracy. The code is available at \href{https://github.com/oneHuster/MixupE}{https://github.com/oneHuster/MixupE}.
\end{abstract}

\section{Introduction}
Deep Neural Networks (DNNs) represent a class of very powerful function approximators, and large-scale DNNs have achieved state-of-the-art performance in many application areas such as computer vision \citep{krizhevsky2012imagenet}, natural language understanding \citep{devlin2018bert}, speech recognition \citep{hinton2012deep}, reinforcement learning \citep{silver2016mastering}, and natural sciences \citep{Jumper2021}. In a supervised learning setting, DNNs are typically trained to minimize their average error on the training samples. This training principle is known as Empirical Risk Minimization (ERM) \citep{vapnick1998statistical}.



Although being a simple training principle, training neural networks with ERM has a major problem: in the absence of regularization techniques, instead of learning meaningful concepts, neural networks trained with ERM are prone to \textit{memorize} training data \citep{arpit}. This results in poor generalization to test samples, which come from a distribution slightly different from the training samples. 
To address this limitation of ERM, Mixup \citep{zhang2017mixup} has recently been proposed as an alternative training principle. In a nutshell, instead of training a neural network on individual samples and their corresponding outputs, Mixup trains a neural network on the linear interpolation of the samples and the corresponding linear interpolation of the outputs. It fosters a smoother decision boundary and reduces the risk of overfitting. Therefore, understanding its implicit regularization helps shed light on generalization. 

Mathematically, let us suppose that $\bx_i$ and  $\bx_j$ are input vectors corresponding to two randomly drawn samples $i$ and $j$ from the training distribution, and $\by_i$ and $\by_j$ are their one-hot encoded labels. Then, Mixup constructs a training sample as $\tx = \lambda \bx_i + (1-\lambda)\bx_j$ and $\ty = \lambda \by_i + (1-\lambda)\by_j$, where $\lambda \in  [0,1]$. Training with this kind of synthetic samples encourages the model to learn a function where linear interpolation in the input vectors leads to the linear interpolation of the corresponding targets. This kind of constraint limits the model complexity, thus limiting their ability to memorize training samples. 
Mixup can be interpreted as a data-agnostic \textit{data augmentation} technique that does not require expert knowledge to create additional training samples.
Mixup can also be interpreted from the viewpoint of the \textit{Vicinal Risk Minimization} (VRM) principle \citep{chapelle2000vicinal}. In this view, Mixup proposes a generic vicinal distribution based on the interpolation of training samples and their associated targets, and the additional training samples are drawn from such vicinal distribution around each training sample \citep{zhang2017mixup}. 

Despite its simplicity and minimal computation overhead, Mixup and its variants have been shown to achieve state-of-the-art in many tasks such as but not limited to, image classification \citep{yun2019cutmix, kim2020puzzle, faramarzi2020patchup}, object detection \citep{jeong2020interpolationbased}, speech recognition \citep{9054719,Tomashenko}, text classification  \citep{guo2019augmenting,zhang2020seqmix}, and medical image segmentation \citep{panfilov2019improving}. Recently, Mixup was theoretically analyzed and shown to be approximately equivalent to adding a second-order regularization term to the standard loss function \citep{zhang2021does}. However, if the benefit of Mixup can be explained by a second-order regularization, the following natural question arises: why can we not replace Mixup with this second-order regularization directly? Unfortunately, the answer is no because the second-order terms are input-specific, thus yielding a complicated form.

In this paper, we show that Mixup is equivalent to implicitly adding infinitely many regularization terms on the directional derivatives of all orders instead of a complex second-order form as \citep{zhang2021does} for ERM. Our analysis provides a feasible insight to design the regularization in practice.
Based on this novel insight, this paper proposes to explicitly enhance the implicit regularization effect of Mixup on the directional derivatives. 
Instead of computing all infinite regularization terms, 
we efficiently approximate the dominant term using accessible results during each forward propagation, ensuring computational efficiency.
We name this method as \MixupS{} (\textbf{Mixup} \textbf{E}nhanced). Furthermore, we give a generalization guarantee of \MixupS{}, which reveals that it achieves lower complexity compared to vanilla Mixup. Figure~\ref{fig:loss} shows the training and test loss of ERM, Mixup, and \MixupS{}. We can see that \MixupS{} has higher training loss, implying that it works as a stronger regularizer than Mixup and ERM. This subsequently results in better generalization (i.e. lower test loss) than Mixup and ERM.









To understand the  benefits of \MixupS{} empirically, we conduct experiments on a variety of datasets, such as images, tabular data and speech data, using various architectures such as LeNet~\citep{lenet}, VGG~\citep{vgg}, ResNet~\citep{he2016deep}, Vision Transformer \citep[ViT,][]{dosovitskiy2020image}, and CoAtNet~\citep{dai2021coatnet}. In our experiments, we consistently see that \MixupS{} has better generalization error than Mixup and ERM, as well as the robustness of test set deformation.

\begin{figure*}[!t]\centering
\includegraphics[width=0.45\textwidth]{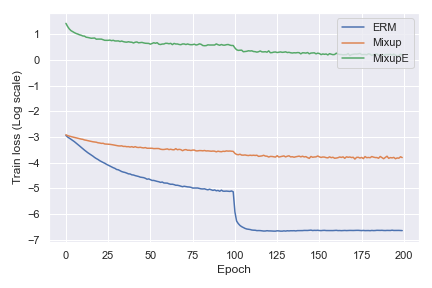}\label{fig:a}
\hfil
\includegraphics[width=0.45\textwidth]{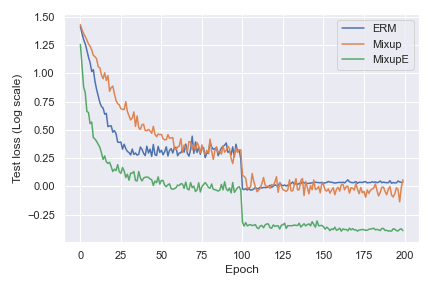}\label{fig:b}
\caption{Comparison of train and test loss of ERM, Mixup, and \MixupS{} trained with Wide-Resnet-28-10 \citep{wrn} architecture. We can see that \MixupS{} has a higher training loss but lower test loss than Mixup and ERM.}
\label{fig:loss}
\end{figure*}

\section{Methods}
In this section, we derive our method with a new mathematical understanding of Mixup.  
We begin in Section~\ref{sec:1} with the notation used to present our theory and method. Mixup is then shown to  implicitly regularize infinitely many directional derivatives of all orders in Section~\ref{sec:2}. This theoretical insight allows us to enhance the regularization effect on the directional derivatives. We also demonstrate that the proposed explicit regularization reduces the algorithmic complexity, thus improving the generalization.  In Section~\ref{sec:3}, we present an algorithm to strengthen the regularization effect of Mixup based on a theoretical derivation and justify the theoretical improvement over Mixup. 


\subsection{Notation} \label{sec:1}
 We denote the input and output pair as 
 $\bx \in\cX\subseteq\bR^d$ and $\by \in\cY\subseteq\bR^C$, respectively. Let $f_\theta(\bx) \in \RR^C $ be the output of the logits (i.e., the last layer before the softmax or sigmoid) of the model parameterized by $\theta$. We use $\ell(\theta, (\bx, \by)) = h(f_\theta(\bx)) - \by\T f_\theta(\bx) $ to denote the loss function where
\begin{equation}
    h(f_\theta(\bx))=\begin{cases}
      \log \left(\sum_{j}\exp(f_\theta(\bx)_{(j)})\right) & \text{Softmax} \\
      \log \left(1+\exp \left(f_\theta(\bx)\right)\right) & \text{Sigmoid}
    \end{cases}     
\end{equation}
where $\exp(\cdot)$ is the exponential function applied to every element. 
Let $g(\cdot)$ be the activation function. We use $\bx_{(i)}$ to index $i$-th element of the vector $\bx$ and $\bx_j$ to represent $j$-th variable in a set. $\mathbf{J}^k$ denotes the $k$-th Jacobian operator.


\paragraph{Mixup} Given a training dataset $S=\{(\bx_i,\by_i)\}_{i=1}^n$ of size $n$ with $\bx_i \in\cX$ and $\by_i\in\cY$, we define the Mixup version of the input and output pair by 
 $\tx_{i,j}(\lambda) = \lambda \bx_i + (1-\lambda) \bx_j$ and $\ty_{i,j}(\lambda) = \lambda \by_i + (1-\lambda) \by_j$ with the Mixup coefficient $\lambda \in [0,1]$. Then, we  denote  the standard empirical loss by 
 $\Ln(\theta,S) = \frac{1}{n}\sum_{i=1}^n l(\theta, (\bx_i, \by_i) )$ and the Mixup loss by 
$$     \Lnmix(\theta, S) := \frac{1}{n^2}\sum_{i,j=1}^n \mathop{\bE}_{\lambda \sim \Beta(\alpha,\beta)}l(\theta, \tx_{i,j}(\lambda), \ty_{i,j}(\lambda))
$$
 where $\Beta(\alpha,\beta)$ represents the beta distribution with its parameters $\alpha,\beta>0$. 
We define a mixture of beta distributions as $\Dcal_\lambda=\frac{\alpha}{\alpha+\beta}\Beta(\alpha+1,\beta)+\frac{\beta}{\alpha+\beta}\Beta(\beta+1,\alpha)$ where the coefficients are the drawing probabilities. Let $a_\lambda=1-\lambda$ and $[n]=\{1,\dots,n\}$.


\subsection{Motivation from implicit regularization of MIXUP} \label{sec:2} 
Here we show a theorem that optimizing the vanilla Mixup loss induces an implicit regularization on directional derivatives of the model $f_\theta$.

\begin{theorem} \label{thm:1}
Let $\ell(\theta, (\bx, \by)) \triangleq h(f_\theta(\bx)) - \by\T f_\theta(\bx)$ be the loss function and $\forall \theta \in \Theta$ functions $f_\theta(\cdot)$ in a $C^K$ manifold. Then the implicit regularization of Mixup is:
\begin{equation}
\begin{aligned}
\Lnmix&(\theta,S) 
 =\Ln(\theta,S) \\ 
 &+ \frac{1}{n}\sum_{i=1}^n \EE_{\substack{\lambda\sim\Dcal_\lambda \\ \bx^\prime \sim \Dcal_X }} \bigg(  \sum_{k=1}^{K} \frac{a_\lambda^{k}}{k!} \mathbf{J}^k_{h}(f_\theta) \Delta_{i}^{\otimes k} \\ 
 &-a_\lambda \by_i\T \Delta_{i}+a_\lambda^{K} \hat \psi_{i, \bxp}(a_\lambda) \bigg)
\end{aligned} 
\end{equation}
where $\mathbf{J}^k_{h}(f_\theta) = g(f_\theta(  \bx_i ))\T$ and 
\begin{equation}
    \Delta_{i}=\sum_{k=1}^K \frac{a_\lambda^{k-1}}{k!} \mathbf{J}^k_{f_\theta}(\bx_i) (\bx^\prime - \bx_{i})^{\otimes k}+a_\lambda^{K-1} \psi_{  i, \bxp }(a_\lambda).
\end{equation}
\end{theorem} 
\begin{remark} $\hat \psi_{i, \bxp}$ and $\psi_{i, \bxp}$ are the remainder terms in Taylor expansion of order $\mathcal{O}(K)$ and with probability $1$, $\lim_{a_\lambda \rightarrow 0}\hat \psi_{i, \bxp}(a_\lambda)=0$, $\lim_{a_\lambda \rightarrow 0}\psi_{i, \bxp}(a_\lambda)=0$.
For cross-entropy loss, given input $\bz \in \RR^d$, $h(\bz)=\log\left(\sum_{j}\exp(\bz_{(j)})\right)$, the derivative of each element is 
$
\frac{\partial h(\bz)}{\partial \bz_{(t)}} = [\sum_{j}\exp(\bz_{(j)})]^{-1} \exp(\bz_{(t)})=g(\bz)_{(t)},
$ 
Similarly, the logistic loss has same derivative form $\frac{\partial h(\bz)}{\partial \bz_{(i)}} = (1+\exp(\bz_{(i)}))^{-1} \exp(\bz_{(i)}) =g(\bz)_{(i)}$. 
Therefore, for both cases, we have the Jacobian w.r.t $\bx_i$ that
$
\mathbf{J}_{h \circ f_\theta} (  \bx_i ) = \mathbf{J}_{h}(f_\theta) \mathbf{J}_{f_\theta}(\bx_i) = g(f_\theta(  \bx_i ))\T\mathbf{J}_{f_\theta}(\bx_i)  \in \RR^{1\times d}.
$ 

\end{remark}

The proof of Theorem~\ref{thm:1} is given in Appendix \ref{app:theorm1}. Theorem~\ref{thm:1} provides the following novel insights: (1) Implicit regularization of Mixup is to add a series of directional derivatives with ascending orders to ERM. (2) To minimize the error brought by remainder terms, we need a large expansion order $K$, or even infinite. Obviously, in this case, explicitly computing the regularizer involves $\mathcal{O}(K^2)$ high-order derivative terms and thus suffers a heavy computational burden. Therefore, instead of replacing Mixup with all explicit regularizers, it is more advantageous to retain Mixup with an extra regularization as it provides a computationally efficient alternative. 


In this view, Theorem~\ref{thm:1} provides a theoretical motivation to further improve Mixup by enhancing its regularization effect in terms of directional derivatives $D_{\theta, S}^1, ..., D_{\theta, S}^K$. For computational efficiency, we propose to strengthen the first-order term while letting Mixup implicitly take care of the higher-order terms.
In Theorem~\ref{thm:1}, the regularization effect of Mixup on the first-order directional derivatives ($k=1$) is captured by$
D_{\theta, S}^1 := \frac{1}{n}\EE_{\substack{\lambda\sim\Dcal_\lambda}} [a_\lambda]\sum_{i=1}^n q(\bx_i),
$
where 
\begin{align} \label{eq:1}
q(\bx_i)=(g(f_\theta(  \bx_i ))-\by_i   )\T \mathbf{J}_{f_\theta}(\bx_i )(\EE[\bxp]- \bx_{i})  
\end{align}
and $\mathbf{J}_{f_\theta} (\bx_i )\in \RR^{C \times d}$. Since $0< a_\lambda < 1$ and small enough, the first-order derivative $D_{\theta, S}^1$ dominates the rest terms $D_{\theta, S}^k$ ($k>1$). Therefore, it turns out to be a stitch in time saves nine if we only use $D_{\theta, S}^1$ as the regularization term. Unfortunately, computing Jacobian in deep models at each iteration step is still time-consuming. 
Furthermore, $q(\bx_i)$ can be approximated by 
\begin{align} \label{eq:2}
\hat{q}(\bx_i)=(\by_i-g(f_\theta(\bx_i)))\T f_\theta(\bx_i),
\end{align}
which lessens the computational burden by removing the derivatives of $f_\theta$. This approximation $q(\bx_i)= \hat{q}(\bx_i)$ holds true when $\EE_{\substack{\bxp\sim \Dcal_X}}[\bxp ]=0$ and $\mathbf{J}_{f_\theta}(\bx_i )\bx_i=f_\theta(\bx_i)$. To this end, we can normalize the training dataset with zero means to realize the first condition. The second condition can be guaranteed from the linear model, i.e. deep neural networks with ReLU activation. Thus we have the approximation that $q(\bx_i) \approx \hat q(x_i)$ and $\mathbf{J}_{f_\theta} (\bx_i )\bx_i\approx f_\theta(\bx_i)$.


However, there is an issue of negativity in the first-order regularization term $D^1_{\theta, S}$ of Mixup. Let $\alpha_{k,i}=(g(f(  \bx_i )_{(k)})-\by_i) \zeta_{k,i}$, (\ref{eq:1}) can be rewritten as
\begin{align}\label{eq:3}
q(\bx_i)=\sum_{j=1}^C \alpha_{j,i} \|\mathbf{J}_{f_\theta} (\bx_i)_{(j)}\|_2 \|\EE_{\substack{\bxp\sim \Dcal_X}}[\bxp ]- \bx_{i}\|_2 
\end{align}
where $f(\bx_i)_{(j)}$ is the $j$-th coordinate of $f(\bx_i)$ and coefficient $\zeta_{j,i}$ is the cosine similarity between $j$-th row vector $\mathbf{J}_{f_\theta} (\bx_i)_{(j)}$ and $\EE_{\substack{\bxp\sim \Dcal_X}}[\bxp ]- \bx_{i}$. If $\alpha_{j,i}$ is positive, then Mixup tends to minimize all first-order directional derivatives  $\|\mathbf{J}_{f_\theta} (\bx_i)_{(j)}\|_2, j \in [C]$. However, if $\alpha_{j,i}$ is negative, Mixup
has an unintended effect of maximizing $\|\mathbf{J}_{f_\theta} (\bx_i)_{(j)}\|_2$. Figure \ref{fig: min_alpha} shows that the minimum values of $\alpha$ are negative for some sample $\bx_i$ and coordinate $k$ in the initial phase of Mixup training. We show these values for both Preactresnet18 and Preactresnet50 architectures on the CIFAR-10 dataset.



\begin{figure}[!ht]
  \centering
    \includegraphics[width=0.9\linewidth]{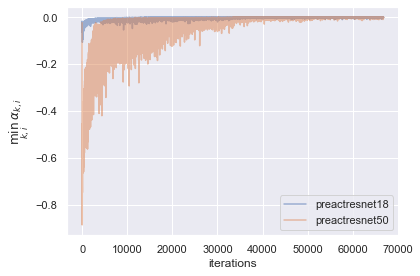}
  \caption{ Minimum value of $\alpha$ over the coordinate $k$ and sample $i$ for different iterations during the training.}
  \label{fig: min_alpha}
\end{figure}

\begin{algorithm}[!t]
\caption{Algorithm of \MixupS{}}
\begin{algorithmic}[1]
\Require Model $f_\theta: \mathbb{R}^{d} \to \mathbb{R}^C$. Hyperparameters $\alpha, \beta$ for beta distribution (mean $\overline{\lambda}$). Loss function $\ell$.
\While{Training epochs $< N$}
\For{Batch of data $X \in \mathbb{R}^{m \times d}, Y \in \mathbb{R}^{m \times C}$ in training set of size $m$}
    \State Sample $\lambda \sim \text{Beta}(\alpha, \beta)$
    \State Mixup data with $\tilde{X}, \tilde{Y} \gets \lambda (X, Y) + a_{\lambda}\text{Permute}(X, Y)$
    \State Mixup Loss $L^{mix}_{n}(\theta, X) = \ell(f_\theta(\tilde{X}), \tilde{Y})$
    \State Compute first-order directional derivatives that $\hat{q}(X) = f_\theta(X) \otimes (Y - \text{Softmax}(f_\theta(X)))$ 
    \State Get additional loss $R(\theta, X)$ via (\ref{eq:additional_term}) 
    \State $\mathcal{L} \gets \hat{\eta} \left(L^{mix}_n(\theta, S)+ \eta R(\theta, S) \right)$ via (\ref{eq:joint_loss})
    \State Optimize parameters $\theta$ with loss $\mathcal{L}$
\EndFor
\EndWhile
\end{algorithmic}
\label{alg:1}
\end{algorithm}

\subsection{Proposed MixupE} \label{sec:3} 


To avoid the unintended effect of maximizing $\|\mathbf{J}_{f_\theta} (\bx_i)_{(j)}\|_2$ in the initial phase of Mixup training, the proposed method uses the following alternative form to ensure the positivity and strengthen the vanilla Mixup: 
\begin{align}\label{eq:additional_term}
 R(\theta, S)=  \frac{\EE_{\substack{\lambda\sim\Dcal_\lambda}} [a_\lambda]}{n} \sum_{i=1}^{n} |\tilde q(\bx_i)|,
\end{align}
where $\tilde q=q$ for the accurate version and $\tilde q = \hat q$ for the approximate version. The functions $q$ and $\hat q$ are defined in equations \eqref{eq:1} and \eqref{eq:2}. The approximate version does not require computation of the directional derivatives, and the additional computational cost is negligible because $f_\theta(\bx_i)$ is known for the original loss and the value of $\EE_{\substack{\lambda\sim\Dcal_\lambda}} [a_\lambda]$ is a fixed number over epochs: e.g., $\EE_{\substack{\lambda\sim\Dcal_\lambda}} [a_\lambda]=1-\frac{\alpha+1}{2\alpha+1}$ when $\Beta(\alpha, \alpha)$ is used for Mixup. 
To justify the rationality of our proposed regularization, we consider a degenerate case. Given a mixup distribution $\Beta(\alpha, \beta)$, if $\alpha, \beta \rightarrow 0$, we have that $R(\theta, S) \rightarrow 0$ as $\EE_{\substack{\lambda\sim\Dcal_\lambda}} [a_\lambda] \to 0$. In this case, Mixup behaves closely to ERM (i.e., mixup coefficient $\lambda$ goes to $0$). Then additional term $R(\theta, S)$ vanishes and is consistent with the behavior of vanilla Mixup. However, introducing $R(\theta, S)$ brings the artifacts when computing the training loss. In line with the original mixup training, we rescale the loss magnitude as before. Overall, we 
propose an explicit regularization $R(\theta, S)$ for vanilla Mixup to strengthen the implicit effect, namely \MixupS{}, where the total loss is defined as: 
\begin{align}\label{eq:joint_loss}
\Lcal(\theta, S) &:= \hat{\eta} \left(L^{mix}_n(\theta, S)+ \eta R(\theta, S) \right), \\
\hat{\eta} &= \frac{|L^{mix}_n(\theta, S)|}{|L^{mix}_n(\theta, S)+ \eta R(\theta, S)|}, 
\end{align}         
where $\hat{\eta}$ is a scaling factor that depends on the magnitudes of $L^{mix}_n(\theta, S)$ and $R(\theta, S)$. Note that $\eta >0$ is the only hyperparameter of the proposed method.

The whole algorithm of \MixupS{} is shown in Algorithm ~\ref{alg:1}. Extending Mixup to \MixupS{} requires one additional forward pass on the \textit{original} (non-mixed) sample for computing the additional loss term and requires one additional hyperparameter $\eta$ in comparison to Mixup. We note that Algorithm ~\ref{alg:1} shows how to apply \MixupS{} when the training sample is in the form of a fixed-shape tensor (for example, images or tabular data). When applying \MixupS{} to training samples with irregular (not fixed) topology, such as graphs, sequences, and trees, we first need to project the input samples to a fixed shape hidden states using an encoder network. After this projection, \MixupS{} can be applied in the usual form.

\paragraph{Theoretical justification: } To validate the rationality of the \MixupS{}, we justify the generalization improvement over Mixup.
Regarding the vanilla Mixup as an original unconstrained problem, for some $\gamma > 0 $, then the constraint $\Theta$ in the dual problem of \MixupS{} will be 
\begin{equation}
    \Theta = \{\bx \to f_\theta(\bx) | \sup_{\bx} | \hat{q}(\bx)| \leq \gamma \}.
\end{equation}
Following \cite{zhang2021does}, we consider Generalized Linear Model (GLM) $h(f_\theta(\bx)) = A(\theta\T \bx)$ for \MixupS{}. 
For simplicity, let $y \in \{0, 1\}$ and data is bounded $\| \bx\|^2 \leq \mathcal{X}$. 
And we define the expected risk of \MixupS{} as:
\begin{equation}
    \tilde{\mathcal{L}}(\theta) := \EE_{S}  \mathcal{L}(\theta, S) 
\end{equation}
Then we proved the following generalization gap:
\begin{theorem}\label{thm:2} Suppose $A(\cdot)$ is $L_A$-Lipchitz continuous, $\mathcal{X}, \mathcal{Y}$ and $\Theta$ are all bounded, the constraint of GLM satisfies $\sup_{\bx} | \hat{q}(\bx)| \leq \gamma$
then there exist constant $B>0$, such that for all $\theta \in \Theta$,
 we have
\begin{equation}
\begin{aligned}
        \tilde{\mathcal{L}}(\theta) 
    & \leq \hat{\eta} L^{mix}_n(\theta, S) +  \frac{ 2 \hat{\eta} \eta L_A \gamma\mathcal{X} }{\sqrt{n} (1+L_A)} \\
    & +  B \sqrt{\frac{\log(1/\delta)}{2n}}
\end{aligned}
\label{eq:complexity}
\end{equation}
with probability at least $1-\delta$.
\end{theorem}
See proof in Appendix \ref{app:theorem2}. In the context of Rademacher complexity, the smaller size of function classes will guarantee better generalization, which is reflected by the \emph{complexity} exemplified as the second term in (\ref{eq:complexity}). 
It is worth noting that for vanilla Mixup, a general parameter constraint can be set, such as $\hat{\Theta}: \|\theta\|^2\leq \xi$. In this case, the \emph{complexity} term can be as high as $L_A \sqrt{\xi \cdot \text{MaxNorm}(\bx) /n}$, depending on the Lipschitz constant $L_A$ while \MixupS{} does not. In general, all the inputs fall in a ball with radius $\mathcal{X}$, which can be normalized as $\mathcal{X} \leq 1$. Therefore, the $\mathcal{X}$ will not lead to a loose bound compared to the standard constraint $\hat{\Theta}$.
In summary, these observations indicate that the regularized version of Mixup, \MixupS{}, leads to better generalization due to the stronger constraint imposed on the parameter space. This validates the rationale behind \MixupS{} and highlights its efficacy in promoting improved generalization performance. Further details and a comprehensive discussion of these findings are in Appendix \ref{app:comparison}.

\section{Related Work}\label{sec:related_work}

%
Mixup \citep{zhang2017mixup,bclearning} and its numerous variants have seen remarkable success in supervised learning problems, as well as other problems such as semi-supervised learning \citep{ict, mixmatch}, unsupervised learning using autoencoders \citep{amr, acai}, adversarial learning \citep{iat,Lee2020AdversarialVM,Pang*2020Mixup}, graph-based learning \citep{verma2020graphmix}, computer vision \citep{yun2019cutmix, jeong2020interpolationbased, panfilov2019improving, faramarzi2020patchup}, natural language \citep{guo2019augmenting,zhang2020seqmix} and speech \citep{9054719,Tomashenko}.

Mixup \citep{zhang2017mixup} creates synthetic training samples by linear interpolation in the input vectors and their corresponding labels. The follow-up work of Mixup can be categorized into two main categories: (a) methods that propose a \textit{non-linear} interpolation in the input vector (or hidden vectors): examples of this category include  \citet{yun2019cutmix,faramarzi2020patchup, kim2020puzzle, zhang2020seqmix, verma2021towards}. (b) methods that extend linear interpolation based Mixup training to various learning paradigms or applications: examples of this class include Mixup based training for supervised learning \citep{manifold_mixup}, semi-supervised learning \citep{ict, mixmatch}, for adversarial training \citep{iat}, for node classification \citep{verma2020graphmix}, and for natural language processing \citep{guo2019augmenting}. The method proposed in this work can be applied to any of the methods in the latter category, and we leave experimental evaluation of \MixupS{} applied to these methods for future work.



 Theoretically, Mixup has been analyzed by \citet{zhang2021does}, in which the authors show that it is approximately equivalent to adding a second-order regularizer to improve robustness and generalization. However, they did not propose a method based on the theory. In contrast, this paper shows that it is  equivalent to adding infinite regularizations on the directional derivatives of all orders and uses this theory to propose a new method. 

\section{Experiments}
We present a range of experiments to back up the methodological claims, demonstrate versatility across benchmark problems, and show practical applicability on images, tabular data, and speech problems.

\subsection{Image Datasets}
\label{sec:image_dataset}

\begin{table*}[!ht]
           \centering
           \caption{Classification errors on (a) CIFAR-10 and (b) CIFAR-100.
                    Standard deviations over five repetitions.
                    Best performing methods in {\bf bold}.
                    }
           \label{tab:supervised}
           \begin{subtable}{\columnwidth}
               \caption{CIFAR-10}
               \label{tab:cifar10}    
               \setlength{\tabcolsep}{22pt}
               \begin{tabular}{lrr} 
               \toprule

               PreActResNet50 & \shortstack[l]{Test Error (\%)}  \\
               \midrule
               ERM & \nolight{\g{4.71}{0.062}} \\ 
               \shortstack[l]{\inputmixup{}} & \nolight{\g{4.53}{0.041}} \\ 
               \shortstack[l]{\MixupS{}} & \g{\textbf{3.53}}{0.047} \\  
               \midrule
               PreActResNet101 \\
               \midrule
               ERM & \nolight{\g{4.21}{0.069}} \\ 
               \shortstack[l]{\inputmixup{}} & \nolight{\g{4.43}{0.049}} \\ 
               \shortstack[l]{\MixupS{}} & \g{\textbf{3.35}}{0.049} \\  
               \midrule
               Wide-Resnet-28-10 \\
               \midrule
               ERM & \nolight{\g{4.24}{0.101}} \\ 
               \shortstack[l]{\inputmixup{} } & \nolight{\g{3.03}{0.091}} \\ 
               \shortstack[l]{\MixupS{}} & \g{\textbf{2.94}}{0.048} \\  
               \bottomrule
               \end{tabular}
           \end{subtable}%
           \hfill
           \begin{subtable}{\columnwidth}
               \caption{CIFAR-100}
               \label{tab:cifar100}   
               \setlength{\tabcolsep}{22pt}
               \begin{tabular}{lrr} 
               \toprule

               PreActResNet50 & \shortstack[l]{Test Error (\%)}  \\
               \midrule
               ERM & \nolight{\g{24.68}{0.349}} \\ 
               \shortstack[l]{\inputmixup{}} & \nolight{\g{23.03}{0.471}} \\ 
               \shortstack[l]{\MixupS{}} & \g{\textbf{20.23}}{0.507} \\  
               \midrule
               PreActResNet101 \\
               \midrule
               ERM & \nolight{\g{23.20}{0.362}} \\ 
               \shortstack[l]{\inputmixup{}} & \nolight{\g{23.05}{0.383}} \\ 
               \shortstack[l]{\MixupS{}} & \g{\textbf{18.86}}{0.376} \\  
               \midrule
               Wide-Resnet-28-10 \\
               \midrule
               ERM & \nolight{\g{22.20}{0.108}} \\ 
               \shortstack[l]{\inputmixup{} } & \nolight{\g{19.38}{0.113}} \\ 
               \shortstack[l]{\MixupS{}} & \g{\textbf{17.12}}{0.111}\\  
               \bottomrule
               \end{tabular}
           \end{subtable}%
\end{table*}

{
\begin{table*}[!ht]
           \centering
           \caption{Classification error on SVHN and classification accuracy on Tiny-Imagenet.
                    Standard deviations over five repetitions.
                    Best performing methods in {\bf bold}.
                    }
           \label{tab:supervised_svhn_tiny_imagenet}
           \begin{subtable}{\columnwidth}
               \caption{Classification Error on SVHN}
               \label{tab:svhn}
               \setlength{\tabcolsep}{22pt}
               \begin{tabular}{lrr} 
               \toprule

               PreActResNet50 & \shortstack[l]{Test Error (\%)}  \\
               \midrule
               ERM & \nolight{\g{2.80}{0.201}} \\ 
               \shortstack[l]{\inputmixup{}} & \nolight{\g{2.65}{0.017}} \\ 
               \shortstack[l]{\MixupS{}} & \g{\textbf{2.42}}{0.021} \\  
               \midrule
               PreActResNet101 \\
               \midrule
               ERM & \nolight{\g{2.95}{0.019}} \\ 
               \shortstack[l]{\inputmixup{}} & \nolight{\g{2.79}{0.015}} \\ 
               \shortstack[l]{\MixupS{}} & \g{\textbf{2.35}}{0.019} \\  
               \midrule
               Wide-Resnet-28-10 \\
               \midrule
               ERM & \nolight{\g{2.82}{0.049}} \\ 
               \shortstack[l]{\inputmixup{} } & \nolight{\g{2.48}{0.117}} \\ 
               \shortstack[l]{\MixupS{}} & \g{\textbf{2.29}}{0.168} \\  
               \bottomrule
               \end{tabular}
           \end{subtable}%
           \hfill
           \begin{subtable}{\columnwidth}
               \caption{Classification Accuracy on Tiny-ImageNet}
               \label{tab:tinyimagenet}   
               \setlength{\tabcolsep}{14pt}
               \begin{tabular}{lrr} 
                \toprule
                PreActResNet18 & top-1 & top-5\\
                \midrule
                ERM & \g{54.97}{0.52} & \g{72.71}{0.48} \\ 
                Mixup  & \g{54.64}{0.43} & \g{72.53}{0.51} \\
                \MixupS{}  & \g{\textbf{62.21}}{0.39} & \g{\textbf{82.09}}{0.41} \\
                \midrule
                 PreActResNet34 &  & \\
                \midrule
                ERM & \g{57.25}{0.48} & \g{72.58}{0.53} \\ 
                Mixup  & \g{57.79}{0.39} & \g{76.15}{0.42} \\
                \MixupS{}  & \g{\textbf{65.37}}{0.31} & \g{\textbf{83.77}}{0.35} \\
                \midrule
                 PreActResNet50 &  & \\
                \midrule
                ERM & \g{55.91}{0.61} & \g{73.50}{0.57} \\ 
                Mixup  & \g{54.86}{0.46} & \g{73.11}{0.43} \\
                \MixupS{}  & \g{\textbf{67.22}}{0.38} & \g{\textbf{85.14}}{0.36} \\
                \bottomrule
                \end{tabular}
           \end{subtable}%
\end{table*}

\textbf{For small-scale image datasets}, we consider the CIFAR-10, CIFAR-100, SVHN, and Tiny-ImageNet. We run our experiments using a variety of architectures, including PreActResNet18, PreActResNet34, PreActResNet50, PreActResNet101 \citep{he2016deep}, and Wide-Resnet-28-10 \citep{wrn}.

Throughout our experiments, we use SGD+Momentum optimizer with batch-size 100, learning rate 0.1, momentum 0.9 and  weight-decay  $10^{-4}$, with step-wise learning rate decay. We train all the networks for all the datasets for 200 epochs, and the learning rate is  annealed by a factor of 10 at epochs 100 and 150.

\textbf{Hyperparameters $\alpha$ and $\eta$:} For Mixup on the CIFAR-10 and CIFAR-100 datasets, we used the value $\alpha=1.0$ as recommended by \citet{zhang2017mixup}. For Mixup on SVHN and Tiny-ImageNet datasets, we experimented with the $\alpha$ values $1.0$ and $0.2$, respectively, as recommended by \citet{manifold_mixup}. 
We experimented with the $\eta \in \{0.0001, 0.001, 0.01, 0.1\}$ and obtained the best results using $\eta = 0.001$ for most of the experiments and using $\eta = 0.0001$ for the remaining experiments. 

\textbf{Results:} We show results for the CIFAR-10 (Table~\ref{tab:cifar10}), CIFAR-100 (Table~\ref{tab:cifar100}), SVHN (Table~\ref{tab:svhn}), and Tiny-ImageNet (Table~\ref{tab:tinyimagenet}) datasets. We see that \MixupS{} consistently outperforms baseline methods ERM and Mixup across all the datasets and architectures.

\textbf{Sensitivity to hyperparameter $\eta$}: To validate that the method is not overly sensitive to the newly introduced hyperparameter $\eta$, we conducted experiments for \MixupS{} with the value of $\eta \in \{0.0001, 0.001, 0.01, 0.1\}$ with Preactresnet50 architecture and the CIFAR-100 dataset. This experiment was repeated five times with different initializations. We got the mean test error (in \%) of $20.23, 20.84, 21.01, 20.87$ for the $\eta$ values of $0.0001, 0.001, 0.01, 0.1$, respectively, vs the mean test-error $23.03$ for Mixup. This suggests that the proposed method \MixupS{} is not overly sensitive to the hyperparameter $\eta$ and works better than Mixup for a large range of $\eta$ values.

\textbf{For a large-scale image classification dataset}, we consider ImageNet~\citep{deng2009imagenet}, using three architectures: ResNet~\citep{he2016deep}, Vision Transformer \citep[ViT,][]{dosovitskiy2020image}, and CoAtNet~\citep{dai2021coatnet}.
In particular, we use ResNet-50, ViT-B/16, and CoAtNet-0.
We choose these architectures for experiments because they are fast to train, and they respectively represent the three families of image classification models: convolution-based models, attention-based models, and hybrid models.

\textbf{Except for the Mixup-related hyperparameters $\alpha$ and $\eta$, all training hyperparameters for these models follow their original paper.} Specifically, all models are trained and evaluated at the resolution of 224x224.
Our ResNet-50 is trained with a SGD+Momentum with the momentum coefficient of 0.9, while our ViT-B/16 and CoAtNet-0 are both trained with AdamW~\citep{loshchilov2018decoupled}, $\beta_1=0.9$ and $\beta_2=0.99$.
An $L_2$ weight decay of $10^{-4}$ is applied to our ResNet-50, while the larger weight decays of $0.05$ and $0.3$ are applied to our CoAtNet-0 and ViT-B/16, respectively.

All models were trained for 100K steps, with a global batch size of 4096.
Throughout these 100K training steps, the learning rate starts from 0 and warms up linearly to its peak value -- which is 1.6 for ResNet-50 and 0.001 for ViT and CoAtNet -- and then decreases to 1/1000 times the peak value following the cosine schedule.
For models with batch normalization, i.e., ResNet-50 and CoAtNet-0, the batch statistics during training are computed globally.
We also apply a Polyak moving average with the rate of $0.999$ on all parameters, \textit{including} the batch normalization cumulative statistics in the case of ResNet-50 and CoAtNet-0.

\textbf{Hyperparameters $\alpha$ and $\eta$:} we use the Mixup rate $\alpha = 0.2$ for ResNet-50, following the suggestions from~\citet{zhang2017mixup}.
Since ViT-B/16 and CoAtNet-0 were invented after Mixup, we first tune the value of $\alpha$ and found that $\alpha=0.4$ offers a sweet spot for these models.
Fixing $\alpha=0.2$ for ResNet-50 and $\alpha=0.4$ for ViT-B/16 and CoAtNet-0, we then tune the values for $\eta$ with ResNet-50.
We find that $\eta=10^{-3}$ is best for ResNet-50 while the smaller value of $\eta=5 \times 10^{-4}$ is best for ViT-B/16 and CoAtNet-0.

\textbf{ImageNet results:} Table~\ref{tab:imagenet} presents our results. We observe that \MixupS{} consistently outperforms Mixup across the three architectures in our experiments.
Notably, the gains of \MixupS{} -- in terms of top-1 accuracy -- are larger for ViT-B/16 and CoAtNet-0 than for ResNet-50,~\ie~+0.7 and +0.8 compared to +0.5, even though the top-1 accuracy of the Mixup baselines for ViT-B/16 and CoAtNet-0 are higher. We note that other extensions of Mixup, such as CutMix ~\citep{yun2019cutmix}, PuzzleMix ~\citep{kim2020puzzle}, and PatchUp ~\citep{faramarzi2020patchup} use \textit{non-linear} mixing of samples, thus they are not directly comparable with \MixupS{}. We leave an experimental comparison of \MixupS{} with these methods using a common implementation scheme (architecture and training/validation protocol) as future work.

\begin{table}[!ht]
\caption{\label{tab:imagenet}ImageNet accuracy of various models. \\ Each experiment was run for 3 times.}
\centering
\resizebox{0.9\linewidth}{!}{ %
\begin{tabular}{lccc}
  \toprule
  \textbf{Models} &
  \textbf{MixUp Type} &
  \textbf{Top-1} &
  \textbf{Top-5} \\
  \midrule
  \multirow{3}{*}{ResNet-50}
            & None    & \g{76.2}{0.5} & \g{93.6}{0.3} \\
            & MixUp   & \g{77.2}{0.2} & \g{94.0}{0.1} \\
            & \MixupS & \g{\textbf{77.7}}{0.2} & \g{\textbf{94.4}}{0.1} \\
  \midrule
  \multirow{3}{*}{ViT-B/16}
            & None    & \g{79.1}{0.2} & \g{95.1}{0.1} \\
            & MixUp   & \g{79.7}{0.2} & \g{95.4}{0.1} \\
            & \MixupS & \g{\textbf{80.4}}{0.1} & \g{\textbf{95.8}}{0.1} \\
  \midrule
  \multirow{3}{*}{CoAtNet-0}
            & None    & \g{79.8}{0.2} & \g{95.1}{0.1} \\
            & MixUp   & \g{80.8}{0.3} & \g{95.5}{0.1} \\
            & \MixupS & \g{\textbf{81.6}}{0.2} & \g{\textbf{95.7}}{0.2} \\
  \bottomrule
\end{tabular}
} %
\end{table}

\textbf{A note on implementation and runtime.} Despite these improvements, \MixupS{} requires twice as many forward passes as normal Mixup.
As shown in Algorithm \ref{alg:1}, the extra computation stems from the forward pass through the original (non-mixed) samples $X$,~\ie~$f_\theta(X)$. 
For larger models running on ImageNet, this cost can lead to significantly slower experiment time.
To alleviate the computational burden, we ``batch'' this extra pass through the non-mixup data into the pass through the mixup data.
Thanks to this trick, our implementation of \MixupS{} is only 1.3 times slower than Mixup.

\subsection{Tabular Datasets}
We consider a number of tabular environments drawn from the UCI dataset~\citep{lichman2013uci}, namely \textit{Arrhythmia, Letter, Balance-scale, Mfeat-factors, Mfeat-fourier, Mfeat-karhunen, Mfeat-morphological, Mfeat-zernike, CMC, Optdigits, Pendigits, Iris, Mnist\_784, Abalone} and \textit{Volkert}. 

We consider the same setting as~\citet{zhang2017mixup}, where the network is a fully-connected multi-layer perceptron (MLP) with two hidden layers, each with 128 dimensions, with ReLU activations for non-linearity. We train this network with the Adam optimizer using the cross-entropy loss with the default learning rate of $0.001$ and a batch size of 100, for 25 epochs. We feed in the categorical part of the data as one-hot inputs, and for any samples with missing features in the dataset, we fill it with the mean (for continuous) or mode (for discrete) of those features.

\begin{table}[!ht]
    \caption{ Classification Test Error (\%) on tabular datasets from UCI repository. Results are averaged over five trials.}
    \centering
    \renewcommand{\arraystretch}{1.25}
    \begin{tabular}{l  c c c}
    \toprule
    \textbf{Dataset}  & \multicolumn{3}{c}{\textbf{Method}} \\
    &  ERM & Mixup & \MixupS{} \\
    \midrule
Arrhythmia  & \g{\textbf{34.60}}{3.10} & \g{35.49}{3.88} & \g{34.85}{3.99} \\
Letter & \g{4.56}{0.27} & \g{\textbf{3.71}}{0.18} & \g{4.04}{0.20} \\
Balance-scale  & \g{3.87}{1.03} & \g{3.70}{1.00} & \g{\textbf{3.68}}{0.97} \\
Mfeat-factors  & \g{2.74}{0.81} & \g{\textbf{2.44}}{0.42} & \g{2.56}{0.64} \\
Mfeat-fourier  & \g{17.69}{1.76} & \g{17.80}{1.56} & \g{\textbf{17.57}}{1.60} \\
Mfeat-karhunen  & \g{3.74}{0.58} & \g{3.06}{0.29} & \g{\textbf{2.47}}{0.32} \\
Mfeat-morph  & \g{25.00}{2.10} & \g{\textbf{24.62}}{1.83} & \g{24.66}{1.30} \\
Mfeat-zernike  & \g{17.58}{1.72} & \g{\textbf{15.19}}{1.73} & \g{15.55}{0.62} \\
CMC  & \g{45.77}{1.49} & \g{46.67}{1.83} & \g{\textbf{45.42}}{2.05} \\
Optdigits  & \g{1.48}{0.19} & \g{\textbf{1.15}}{0.21} & \g{1.33}{0.14} \\
Pendigits  & \g{1.03}{0.25} & \g{0.76}{0.19} & \g{\textbf{0.72}}{0.16} \\
Iris  & \g{9.06}{7.01} & \g{8.14}{6.48} & \g{\textbf{7.29}}{6.95} \\
Mnist\_784  & \g{2.83}{0.11} & \g{2.57}{0.05} & \g{\textbf{2.56}}{0.14} \\
Abalone  & \g{35.05}{0.61} & \g{35.07}{0.69} & \g{\textbf{34.91}}{0.70} \\
Volkert  & \g{33.26}{0.62} & \g{32.74}{0.76} & \g{\textbf{32.54}}{0.61} \\

     \bottomrule
    \end{tabular}
    \label{tab:tabular}
\end{table}

\textbf{Hyperparameters $\alpha$ and $\eta$:} We consider the hyper-parameters from the set $\alpha \in \{0.01, 0.02, 0.05, 0.1$, $ 0.2, 0.5, 1.0, 2.0, 5.0\}$ and $\eta \in \{0.0001, 0.001, 0.01$, $0.1, 1.0\}$ and run five seeds for each of the combinations and algorithms. Then, the value of $(\alpha, \eta)$ is chosen based on the best validation accuracy, corresponding to which we report the test accuracy for that particular dataset.

\textbf{Results:} Table~\ref{tab:tabular} presents our results for a subset of the tabular datasets. We observe that \MixupS{} outperforms the standard Mixup as well as ERM across multiple datasets. Among these 15 datasets, MixupE surpasses the baselines substantially (9 datasets) and achieves comparable performances as the best, such as Arrhythmia, Letter, Mfeat-zernike and Mfeat-morph (4 datasets). On the whole, MixupE has demonstrated considerable improvements over vanilla Mixup by considering the relative improvements of MixupE compared to the standard ERM training scheme.

\subsection{Speech Dataset}

To have a rigorous comparison with \citet{zhang2017mixup}, similar to their work for the speech dataset, we use the Google commands dataset \citep{google_speech_dataset}. This dataset consists of 65000 one-second long utterances of 30 short words, such as \textit{yes, no, up, down, left, right, stop, go, on, off}, by thousands of different people. 30 short words correspond to 30 classes. We preprocess the utterances by first extracting the normalized spectrograms from the original waveform at a sampling rate of 16 kHz, followed by zero-padding the spectrograms to equalize their size at $160 \times 101$. This preprocessing step is exactly the same as \citet{zhang2017mixup}. Furthermore, similar to \citet{zhang2017mixup}, we use  LeNet \citep{lenet} and VGG-11 and VGG-13 \citep{vgg} architectures. We train all the models for 20 epochs using Adam optimizer with a learning rate of 0.001 and batch size of 100.

\begin{table}[!ht]
    \caption{Classification Test Error (\%) on Google Speech Command Dataset \citep{google_speech_dataset}. We run each experiment five times}
    \centering
    \renewcommand{\arraystretch}{1.25}
    \begin{tabular}{c  c c c}
    \toprule
    \textbf{Architecture}  & \multicolumn{3}{c}{\textbf{Method}} \\
    &  ERM & Mixup & \MixupS{} \\
    \midrule
    LeNet  & \g{10.43}{0.052} & \g{10.12}{0.041} & \g{\textbf{10.02}}{0.042} \\
    VGG-11 & \g{6.04}{0.059} & \g{4.63}{0.047} & \g{\textbf{3.93}}{0.050} \\
    VGG-13  & \g{5.77}{0.053} & \g{4.68}{0.039} &  \g{\textbf{3.84}}{0.040} \\
     \bottomrule
    \end{tabular}
    
    \label{tab:google_speech}
\end{table}

\textbf{Hyperparameters $\alpha$ and $\eta$:} 
For all the architectures,  we first find the best value of hyperparameter $\alpha$ for Mixup from the set $\alpha \in \{0.1, 0.2, 0.5, 1.0, 2.0\}$. We observed that $\alpha = 0.2$ works best consistently for all architectures. For \MixupS{}, we used the best $\alpha$ values from Mixup and only fine-tuned the $\eta$ hyperparameter using $\eta \in \{0.001, 0.01, 0.1, 1.0\}$. In our experiments, $\eta=0.01$ works best for all the experiments.

\textbf{Results:} In Table \ref{tab:google_speech}, we observe that \MixupS{} improves the test error of Mixup for different architectures. Moreover, the improvement is more significant for larger architectures such as VGG-11 and VGG-13 than LeNet.

\subsection{Graph Datasets}
For graph classification, we consider the \textit{MUTAG, NCI1, PTC, PROTEINS, IMDB-BINARY} and \textit{IMDB-MULTI} datasets. We use the experimental settings defined in \citet{xu2018powerful} as the baseline system, where Mixup and \MixupS{} are performed after encoding the graph to a fixed dimensional vector, that is, at the graph-level readout stage. Each system here relies on 5 graph neural network layers that give rise to the readout, which a non-linear MLP then operates on. The models are trained for 350 epochs using the Adam optimizer with a learning rate of 0.01, which is halved every 50 epochs. For the hyperparameters, we consider $\alpha \in \{0.01, 0.02, 0.05, 0.1, 0.2, 0.5, 1.0, 2.0, 5.0\}$ and $\eta \in \{0.0001, 0.001, 0.01, 0.1, 1.0\}$. Corresponding to each model setting, we perform 10-fold validation, identify which epoch and hyperparameters give the best test accuracy, and report the algorithm's final mean and standard deviation over the ten folds. We refer the readers to Table \ref{tab:graph_pgnn}, which shows the benefits of using \MixupS{} on the graph datasets.

In conclusion, \MixupS{} outperforms than vanilla Mixup on different types of datasets. The proposed regularizer effectively improve the generalization of Mixup.
\begin{table}[!ht]
    \caption{ Classification Test Error (\%) on graph datasets from the TUDatasets benchmark when following the setup of \citet{xu2018powerful}. Results are obtained from 10-fold validation.}
    \centering
    \renewcommand{\arraystretch}{1.25}
    \begin{tabular}{l  c c c}
    \toprule
    \textbf{Dataset}  & \multicolumn{3}{c}{\textbf{Method}} \\
    &  ERM & Mixup & \MixupS{} \\
    \midrule
MUTAG & \g{10.15}{0.06} & \g{10.67}{0.05} & \g{\textbf{10.06}}{0.06}\\
NCI1 & \g{17.79}{0.02} & \g{18.59}{0.02} & \g{\textbf{17.74}}{0.01}\\
PTC & \g{38.37}{0.09} & \g{\textbf{34.87}}{0.08} & \g{35.50}{0.08}\\
PROTEINS & \g{25.43}{0.04} & \g{24.44}{0.04} & \g{\textbf{23.72}}{0.04}\\
IMDBBINARY & \g{25.60}{0.03} & \g{25.30}{0.03} & \g{\textbf{25.20}}{0.03}\\
IMDBMULTI & \g{50.33}{0.03} & \g{49.27}{0.04} & \g{\textbf{48.53}}{0.03}\\
    \bottomrule
    \end{tabular}
    \label{tab:graph_pgnn}
\end{table}

\begin{table}[!ht]
\caption{Ablation experiments to understand the effect of the additional loss term in Equation \ref{eq:additional_term}. Each experiment was run 5 times.}
\centering
\resizebox{0.9\linewidth}{!}{ %
\begin{tabular}{lc}
  \toprule
  \textbf{Method} &
  \textbf{Test Error} \\
  \midrule
  ERM & 24.68 \\
  Mixup & 23.03 \\
  ERM+additional loss & 22.42 \\
  Mixup+additional loss (\MixupS{}) & 20.23 \\
  \bottomrule
\end{tabular}
\label{tab:ablation}
}
\end{table} %

\begin{table*}[!t]
\caption{Test accuracy on novel deformations. All models are trained on normal CIFAR-100.}
\centering
\begin{tabular}{lccccc}
  \toprule
  \textbf{Test Set Deformation} &
  Mixup ($\alpha=1$) & Mixup ($\alpha=2$) & Manifold Mixup ($\alpha=2$) & Ours ($\alpha=1$)\\
  \midrule
Rotation $U(-20,20)$ &	55.55 &	56.48 &	60.08 &	\textbf{62.23} \\
Rotation $U(-40,40)$ &	37.73 &	36.78 &	42.13 &	\textbf{43.08} \\
Shearing $U(-28.6,28.6)$ &	58.16 &	60.01 &	62.85 &	\textbf{63.94} \\
Shearing $U(-57.3,57.3)$ & 39.34 & 39.70 & \textbf{44.27} & \textbf{43.87} \\
Zoom In (60\% rescale) &	13.75 &	13.12 &	11.49 &	\textbf{15.66} \\
Zoom In (80\% rescale) &	52.18 &	50.47 &	52.70 &	\textbf{54.22} \\
Zoom Out (120\% rescale) &	60.02 &	61.62 &	\textbf{63.59} &	61.39 \\
Zoom Out (140\% rescale) &	41.81 &	42.02 &	\textbf{45.29} &	36.58 \\
  \bottomrule
\end{tabular}
\label{tab:deformation}
\end{table*} %

\subsection{Ablation Experiments}
In \MixupS{}, we have proposed to add an additional loss term derived from the first-order derivative (\ref{eq:additional_term}) to the Mixup Loss \ref{eq:joint_loss}. A natural question arises: what would be the performance of adding this term to the ERM loss? We conduct an ablation study to investigate this question. Specifically, we compare the following four methods : 1) ERM, 2) Mixup, 3) ERM+additional loss, and 4) Mixup+additional loss (\MixupS{}). The test error on the CIFAR-100 dataset using the Preactresnet50 architecture for the abovementioned method is shown in Table \ref{tab:ablation}.

Results in Table \ref{tab:ablation} show that adding the additional loss term of Eqn (\ref{eq:additional_term}) improves the test accuracy in Mixup. This is consistent with our argument in Section \ref{sec:2} that Mixup can have an unintended effect of maximizing $\|\mathbf{J}_{f_\theta} (\bx_i)_{(j)}\|_2$. Furthermore, we observe that Mixup+additional loss ( \MixupS{}) performs better than ERM+additional loss; this indicates that the implicit regularization of higher order directional derivative through Mixup training is important for better test errors, thus justifying our proposed method.

\subsection{Generalization to Novel Deformations}
Following \cite{manifold_mixup}, we also evaluate the robustness of the representations learned by MixupE and compare it to other baselines. For our method, we use the PreActResNet18 only trained with 400 epochs instead of 1200 epochs of Manifold Mixup, which means fewer training epochs were used to obtain our results than other baselines reported in \cite{manifold_mixup}. As shown in Table \ref{tab:deformation}, the results indicate that our method consistently outperforms the other methods in most test set deformations. Specifically, for rotation in the range of $U(-20,20)$ and Zoom In (Rows 1, 5, 6), our method significantly improved over all baselines, which is the highest among all methods. MixupE again outperforms the other methods in Rows 2 and 3 for the rest settings and achieves similar accuracy to the previous SOTA in Row 4. These results suggest that \MixupS{} has a better generalization to novel deformation test data. 

\section{Conclusion and limitations}
In this work, we have theoretically derived a new method to improve Mixup. 
Our theory shows that Mixup is a computationally efficient way to regularize directional derivatives of all orders (see Theorem \ref{thm:1}). Based on this intuition, we propose a new Mixup variant, termed \MixupS{}, a simple and one-line code modification of the original Mixup. Our proposed method is mathematically designed to strengthen the regularization effect of Mixup with a generalization improvement guarantee (see Theorem \ref{thm:2}). Empirically, \MixupS{} outperforms Mixup on several datasets, such as image, tabular, and speech datasets, trained with various networks. 
The improvement in test error is more significant for networks with larger capacities. As a limitation, our method requires one additional forward pass in the network during training than Mixup but only suffers an extra $30\%$ time cost than Mixup. While we only approximate the first-order term for the computational efficiency, our results suggest a promising future research direction to enhance Mixup by studying higher-order terms in Theorem \ref{thm:1}.

\begin{acknowledgements} 
This research/project is supported by the National Research Foundation, Singapore under its AI Singapore Programme (AISG Award No: AISG-GC-2019-001-2A) and by the Google Cloud Research Credits program with the award (6NW8-CF7K-3AG4-1WH1). The computational work for this article was partially performed on resources of the National Supercomputing Centre, Singapore (https://www.nscc.sg).

\end{acknowledgements}

\bibliography{MixupE_129}

\newpage
\appendix
\onecolumn
\section{Notations}
 We denote by $z=(\bx, \by)$ the input and output pair where $\bx \in\cX\subseteq\bR^d$ and $\by \in\cY\subseteq\bR^C$. Let $f_\theta(\bx) \in \RR^C $ be the output of the logits (i.e., the last layer before the softmax or sigmoid) of the model parameterized by $\theta$. We use $\ell(\theta, \bz) = h(f_\theta(\bx)) - \by\T f_\theta(\bx) $ to denote the loss function. 
Let $g(\cdot)$ be the activation function. We use $\bx_{(i)}$ to index $i$-th element of the vector $\bx$ and $\bx_j$ to represent $j$-th variable in a set. The notation list is:
\begin{itemize}
    \item $S = \{\bx_i, \by_i\}_{i \in [n]}$ is the fixed training set while $\bxp$ is the random test sample.
    \item $\ell$ is the loss function for any data point.
    \item $L_n^{mix}(\theta, S)$: empirical risk of Mixup of size $n$ with parameters $\theta$.
    \item $\mathcal{L}$: empirical risk of \MixupS{}.
    \item $\Theta$: the constraint set of parameters $\theta$.
    \item $\mathcal{R}(\Theta, S)$: Empirical Rademacher complexity of set $\Theta$ over training set $S$.
    \item $\mathbf{J}_{a}(b)$: Jacobian matrix of $a$ w.r.t $b$.
\end{itemize}

\section{Proof of Theorem 1}
\label{app:theorm1}
\begin{proof}
For the cross-entropy loss, we have 
\begin{equation}
    \ell(\theta, (\bx, \by)) = -\log\frac{\exp(\by\T f_\theta(\bx))}{\sum_{j}\exp(f_\theta(\bx)_{(j)})}=\log\left(\sum_{j}\exp(f_\theta(\bx)_{(j)})\right)-\by\T f_\theta(\bx)
\end{equation}
where $\by \in \RR^C$ is a one-hot vector.
For the logistic loss, 
\begin{equation}
\ell(\theta,(\bx, \by)) = -\log\frac{\exp(\by f_\theta(\bx))}{1+\exp(f_\theta(\bx))}=\log\left(1+\exp(f_\theta(\bx) \right)-\by f_\theta(\bx).
\end{equation}
Thus, for both cases, we can write 
\begin{equation}
    \ell(\theta,(\bx, \by)=h(f_\theta(\bx))- \by\T f_\theta(\bx)
\end{equation}
where 
$h(\bz)=\log\left(\sum_{j}\exp(\bz_j)\right)$ for the cross-entropy loss and $h(\bz)=\log(1+\exp(\bz))$ for the logistic loss. Using this and equation (9) of \citep{zhang2021does}, we have that 
$$
\Lnmix(\theta,S)=\frac{1}{n}\sum_{i=1}^n \EE_{\lambda\sim\Dcal_\lambda}\bE_{\bxp\sim \Dcal_X} l(\theta,(r_i(\bxp), \by_i)),
$$
where $\Dcal_X$ is the empirical distribution induced by training samples, and
\begin{equation}
    r_i(\bx)=\lambda \bx_i +(1-\lambda) \bx.
\end{equation}
Define $a_\lambda = 1-\lambda$. Then, 
\begin{equation}
    r_i(\bxp) = (1-a_\lambda) \bx_i +a_\lambda \bxp= \bx_i +a_\lambda(\bxp - \bx_{i}).
\end{equation}
Define 
\begin{equation}
    \varphi_{i}(a_\lambda) := f_\theta(  \bx_i + a_\lambda (\bxp - \bx_{i}))
\end{equation}
Assume $f_\theta$ lies in the $C^K$ manifold ($K$-times differentiable), then there exists a function $\psi_{i}$ such that $\lim_{a_\lambda \rightarrow 0}\psi_{ i }(a_\lambda)=0$ and with Taylor expansion at $a_\lambda=0$, we have
\begin{equation}
    \begin{aligned}
    \varphi_{i}(a_\lambda) 
    &= \varphi_{i}(0)+  \sum_{k=1}^K \frac{a_\lambda^{k}}{k!} \varphi^{(k)}_{i}(0)+a_\lambda^{K} \psi_{i}(a_\lambda)  \\
    &= f_\theta(  \bx_i )+  \sum_{k=1}^K \frac{a_\lambda^{k}}{k!} \varphi^{(k)}_{i}(0)+a_\lambda^{K} \psi_{i}(a_\lambda)
    \label{eq:varphi}
    \end{aligned}
\end{equation}
where  $\varphi^{(k)}_i(0)$ is the $k$-th order derivative at $a_\lambda = 0$, $\psi_{i}(a_\lambda)$ is the remainder term:
\begin{equation}
    \psi_{i}(a_\lambda) =  \int_{\RR} \varphi^{(K)}_i(a_\lambda) d a_\lambda - \frac{1}{k!} \varphi^{(K)}_i(0)
\end{equation}
Here, for any $k \in \NN^+$, we have
\begin{equation}
\begin{aligned}
        \varphi^{(k)}_{i}(0) 
        &= \varphi^{(k)}_{i}(a_\lambda)|_{a_\lambda=0} =
     \frac{\partial^{k} f_\theta(\bx_i + a_\lambda (\bxp - \bx_{i}))}{\partial (\bx_i + a_\lambda (\bxp - \bx_{i}))^k} (\bxp - \bx_{i})^{\otimes k}\bigg |_{a_\lambda=0} \\
     &=\frac{\partial^{k} f_\theta(\bx_i)}{\partial (\bx_i)^k} (\bxp - \bx_{i})^{\otimes k} \\
\end{aligned}
\end{equation} 
where $\otimes$ denotes Kronecker product and thus $(\bxp - \bx_{i})^{\otimes k} \in \RR^{d^k}$. We can then rewrite $\varphi^{(k)}_{i}(0)$ as
\begin{equation}
    \varphi^{(k)}_{i}(0) = \mathbf{J}^k_{f_\theta}(\bx_i) (\bxp - \bx_{i})^{\otimes k}
\end{equation}
Plug back into the (\ref{eq:varphi}), we have
\begin{equation}
\begin{aligned}
        f_\theta(  \bx_i + a_\lambda(\bxp - \bx_{i}))
        &=f_\theta(  \bx_i )+  \sum_{k=1}^K \frac{a_\lambda^{k}}{k!} \mathbf{J}^k_{f_\theta}(\bx_i) (\bxp - \bx_{i})^{\otimes k} + a_\lambda^{K} \psi_{i}(a_\lambda) \\
        &= f_\theta(  \bx_i )+  a_\lambda \underbrace{\left( \sum_{k=1}^K \frac{a_\lambda^{k-1}}{k!} \mathbf{J}^k_{f_\theta}(\bx_i) (\bxp - \bx_{i})^{\otimes k} + a_\lambda^{K-1} \psi_{i}(a_\lambda) \right)}_{\Delta_{i}}
\end{aligned}
\end{equation}
Above equation will be
\begin{equation}
\begin{aligned}
\ell(\theta, (r_i(\bx), \by_i)) 
& =\ell[\theta,(  \bx_i +a_\lambda(\bxp - \bx_{i}), \by_i)]\\ 
& =h(f_\theta(  \bx_i +a_\lambda(\bxp - \bx_{i})))-\by_i\T f_\theta(  \bx_i +a_\lambda(\bxp - \bx_{i}))  \\ 
& =h(f_\theta(  \bx_i )+a_\lambda \Delta_{i})-\by_i\T (f_\theta( \bx_i )+a_\lambda\Delta_{i}).
\end{aligned}
\end{equation} 
Analogously, we can define $\hat \varphi^{(k)}_{i}(a_\lambda) := h(f_\theta(  \bx_i )+a_\lambda \Delta_{i})$ and the parallel notation $\hat \psi_{i}(a_\lambda)$, then
\begin{equation}
    h(f_\theta(  \bx_i )+a_\lambda \Delta_{i})=h(f_\theta(  \bx_i ))+  \sum_{k=1}^{K} \frac{a_\lambda^{k}}{k!} \mathbf{J}^k_{h \circ f_\theta}(\bx_i) \Delta_{i}^{\otimes k}+a_\lambda^{K} \hat \psi_{i}(a_\lambda)
\end{equation}
Combining these,
\begin{equation}
   \begin{aligned}
\ell(\theta,( r_i(\bx), \by_i)) 
& = h(f_\theta(  \bx_i ))-\by_i\T f_\theta(  \bx_i )-a_\lambda \by_i\Delta_{i} + \sum_{k=1}^{K} \frac{a_\lambda^{k}}{k!} \mathbf{J}^k_{h \circ f_\theta}(\bx_i) \Delta_{i}^{\otimes k}+a_\lambda^{K} \hat \psi_{i}(a_\lambda) \\ 
& =\ell(\theta,(\bx, \by_i)) - a_\lambda \by_i\T \Delta_{i}+  \sum_{k=1}^{K} \frac{a_\lambda^{k}}{k!} \mathbf{J}^k_{h \circ f_\theta}(\bx_i) \Delta_{i}^{\otimes k}+a_\lambda^{K} \hat \psi_{i}(a_\lambda)
\end{aligned} 
\end{equation}
Thus, the implicit regularization of Mixup can be unfolded as
\begin{equation}
   \begin{aligned}
\Lnmix(\theta, S) 
&=\frac{1}{n}\sum_{i=1}^n \EE_{\lambda\sim\Dcal_\lambda}\bE_{\bx\sim \Dcal_X} l(\theta,( r_i(\bx), \by_i))  \\
& =\Ln(\theta, S)+ \frac{1}{n}\sum_{i=1}^n \EE_{\lambda\sim\Dcal_\lambda}\bE_{\bx\sim \Dcal_X}\left(  \sum_{k=1}^{K} \frac{a_\lambda^{k}}{k!} \mathbf{J}^k_{h \circ f_\theta}(\bx_i) \Delta_{i}^{\otimes k}-a_\lambda \by_i\T \Delta_{i}+a_\lambda^{K} \hat \psi_{i}(a_\lambda) \right),
\end{aligned} 
\end{equation}
 where
\begin{equation}
    \Delta_{i}=\sum_{k=1}^K \frac{a_\lambda^{k-1}}{k!} \mathbf{J}^k_{f_\theta}(\bx_i) (\bxp - \bx_{i})^{\otimes k}+a_\lambda^{K-1} \psi_{  i }(a_\lambda).
\end{equation}
Note that with probability $1$, we have 
$$\lim_{a_\lambda \rightarrow 0} \hat \psi_i(a_\lambda)=0, \lim_{a_\lambda \rightarrow 0} \psi_{i}(a_\lambda)=0$$

\end{proof}

\section{Proof of Theorem 2}
\label{app:theorem2}

The Rademacher generalization bound is widely applied where the empirical Rademacher complexity of a function class $\Theta$ is given by:
\begin{equation}
    \mathcal{R}_n(\Theta, \{\bx_i\}_{i\in [n]})=\mathbb{E}\left[\sup _{\theta \in \Theta} \frac{1}{n} \sum_{i=1}^n f_{\theta} \left(\bx_i\right) \epsilon_i\right]
\end{equation}
where, Rademacher r.v $\epsilon_i$ independently takes values in $\{-1, +1\}$ with equal probability. 
\begin{lemma} (\cite{bartlett2002rademacher}). For any B-uniformly bounded and $L$ Lipchitz function $\zeta$, for all $\phi \in \Phi$, with probability at least $1-\delta$,
$$
\mathbb{E} \zeta\left(\phi \left(\bx_i\right)\right) \leq \frac{1}{n} \sum_{i=1}^n \zeta\left(\phi\left(\bx_i\right)\right)+2 L \mathcal{R}_n(\Phi, S)+B \sqrt{\frac{\log (1 / \delta)}{2 n}}
$$
\label{lemma:1}
\end{lemma}

\begin{proof}
Consider GLM that $h(f_\theta(\bx)) = A(\theta\T \bx)$ and training set $S$, $y \in \{0, 1\}$ and the constraint of $ \Theta = \{\bx \to f_\theta(\bx) | \sup_{\bx} | \hat{q}(\bx) | \leq \gamma \}$ implies that
\begin{equation}
     \sup_{\bx} | \hat{q}_i(\bx) | = \sup_{\bx} | (y - A'(\theta^\top \bx)) (\theta^\top \bx) | \leq \gamma 
\end{equation}
Since $A'(\theta^\top \bx)$, we have 
\begin{equation}
\begin{aligned}
    \gamma  
    & \geq \sup_{\bx}  |y - A'(\theta^\top \bx)| \sqrt{\theta^\top \bx \bx^\top  \theta } \\
    & \geq \sup_{\bx}  |y - A'(\theta^\top \bx)| \sqrt{\theta^\top \Sigma_X  \theta } 
\end{aligned}
\end{equation}
Due to the fact that $A(\cdot)$ is a $L_A$ Lipchitz function, then it's trivial to prove
\begin{equation}
     | A'(\theta^\top \bx) | \leq L_A
\end{equation}
Let $\Sigma_X = \bx \bx^\top $,  $\mathbf{v} =  \Sigma^{1/2}_X  \theta$, then we have
\begin{equation}
    \| \mathbf{v} \|^2 \leq \frac{\gamma^2}{(1+L_A)^2} 
\end{equation}
Denote $\hat{\mathbf{x}}_i = \Sigma^{1/2}_X  \bx_i$, we have the Rademacher complexity $\mathcal{R}(\Theta, S)$ that
\begin{equation}
    \begin{aligned}
\mathcal{R}\left(\Theta, S\right) 
& =\mathbb{E}_{\epsilon} \sup_{\sup_\bx | \hat{q}(\bx) | \leq \gamma} \frac{1}{n} \sum_{i=1}^n \epsilon_i \theta^{\top} \bx_i \\
& \leq \mathbb{E}_{\epsilon} \sup_{ \|\mathbf{v}_i\|_2 \leq \frac{\gamma^2}{(1+L_A)^2}  } \frac{1}{n} \sum_{i=1}^n \epsilon_i  \mathbf{v}_i^\top \hat{\mathbf{x}}_i \\
& \leq \frac{1}{n} \cdot  \frac{\gamma}{1+L_A} \cdot  \sqrt{\mathbb{E}_{\epsilon} \left \| \sum_{i=1}^n  \epsilon_{i} \hat{\mathbf{x}}_i  \right\|^2} \\
& \leq \frac{1}{n} \cdot  \frac{\gamma}{1+L_A} \cdot  \sqrt{ \sum_{i=1}^n   \hat{\mathbf{x}}_i^\top \hat{\mathbf{x}}_i  } 
\end{aligned}
\end{equation}
Consequently, we have 
\begin{equation}
\begin{aligned}
 \mathcal{R}\left(\Theta, S\right) 
 \leq  \frac{1}{n} \cdot  \frac{\gamma}{1+L_A} \cdot  \sqrt{ \sum_{i=1}^n   \mathbf{x}_i^\top \Sigma_X \mathbf{x}_i }   = \frac{1}{\sqrt{n}} \cdot  \frac{\gamma}{1+L_A} \mathcal{X}   
\end{aligned}
\end{equation}
Recall the objective of \MixupS{}, 
\begin{align}
\Lcal(\theta, S) &:= \hat{\eta} \left(L^{mix}_n(\theta, S)+ \eta R(\theta, S) \right) \\
\hat{\eta} &= \frac{|L^{mix}_n(\theta, S)|}{|L^{mix}_n(\theta, S)+ \eta R(\theta, S)|} 
\end{align}  
Define the expected risk of $\hat{L} (\theta) := \EE_{S} L_n^{mix} (\theta, S)$
\begin{equation}
\begin{aligned}
    \hat{L} (\theta)
    & \leq  L^{mix}_n(\theta, S) +  2 L_A \mathcal{R}(\Theta, S) +  B \sqrt{\frac{\log(1/\delta)}{2n}} \\
    & \leq \hat{\eta} L^{mix}_n(\theta, S) +  \frac{ 2 L_A }{\sqrt{n}}   \frac{\gamma\mathcal{X}}{1+L_A}   +  B \sqrt{\frac{\log(1/\delta)}{2n}}
\end{aligned}
\end{equation}
\end{proof}

\subsection{Comparison to vanilla Mixup}
\label{app:comparison}
As a comparison, for vanilla Mixup with parameter space $\hat \Theta = \{ \theta | \|\theta\|_2^2 \leq \xi \}$ and assume $\left\| \bx_i \right\|^2 \leq \mathcal{X}, \forall i \in [n]$ the Rademacher complexity will be
\begin{equation}
\begin{aligned}
    \mathcal{R}(\hat \Theta, S) 
    &= \mathbb{E}_{\epsilon} \sup_{\|\theta\|_2^2 \leq \xi } \frac{1}{n} \sum_{i=1}^n \epsilon_i \theta^{\top} \bx_i \\
    &\leq \frac{1}{n} \mathbb{E}_{\epsilon}  \sup_{\|\theta\|_2^2 \leq \xi } \|\theta\| \sqrt{ \left\| \sum_{i=1}^n \epsilon_i \bx_i \right\|^2 } \\
    &\leq \frac{\sqrt{\xi }}{n} \mathbb{E}_{\epsilon}  \sqrt{ \left\| \sum_{i=1}^n \epsilon_i \bx_i \right\|^2 } \\
    &\leq \frac{\sqrt{\xi }}{n}   \sqrt{  \sum_{i=1}^n \left\| \bx_i \right\|^2 } \\
    & \leq \frac{\sqrt{\xi  \mathcal{X}}}{\sqrt{n}}
\end{aligned}
\label{eq:mixup_rademacher}
\end{equation}
However, if considering normalized input space where $\mathcal{X} = 1$, the condition to have a shrink parameter space is
\begin{equation}
\frac{  L_A \gamma\mathcal{X} }{ (1+L_A) } \leq L_A \sqrt{ \xi \mathcal{X}} \Rightarrow \gamma \sqrt{\mathcal{X}} \leq (1+L_A) \sqrt{\xi} 
\end{equation}
When the data is normalized, it is true that $\sqrt{\mathcal{X}} \leq 1$. Therefore, the above inequality will be 
\begin{equation}
    \gamma \leq (1+L_A) \sqrt{\xi}
\end{equation}
Thus, when the above condition is satisfied, our regularization reduces the norm of parameter space. In general, Lipschitz constant $L_A$ can be large depending on the type of activation function. In summary, shrinking the parameter space by the regularization of \MixupS~~ probably satisfies in most cases.

\section{Implementation}
\label{app:implementation}
The code implementation in PyTorch is shown as Listing \ref{fig: code_block}.
\begin{python}[float=*h, caption={One epoch \MixupS{} training in PyTorch}, label={fig: code_block}]
def beta_mean(alpha, beta):
    return alpha/(alpha+beta)
    
lam_mod_mean = beta_mean(alpha+1, alpha) # mean of beta distribution

# y1, y2 should be one-hot vectors
for (x1, y1), (x2, y2) in zip(loader1, loader2):
    lam = numpy.random.beta(alpha, alpha)
    x = Variable(lam * x1 + (1. - lam) * x2)
    y = Variable(lam * y1 + (1. - lam) * y2)
    loss = loss_function(net(x), y)  # mixup loss
    loss_scale = torch.abs(loss.detach().data.clone())
    f = net(x1)
    b = y1 - torch.softmax(f, dim=1)
    loss_new = torch.sum(f * b, dim=1)
    loss_new = (1.0 - lam_mod_mean) * torch.sum(torch.abs(loss_new)) / batch_size  # additional loss term
    loss = loss + (mixup_eta * loss_new)  # total loss
    loss_new_scale = torch.abs(loss.detach().data.clone())
    loss = (loss_scale / loss_new_scale) * loss  # loss after scaling
    optimizer.zero_grad()
    loss.backward()
    optimizer.step()
\end{python}

\end{document}